\pdfoutput=1

\documentclass[11pt]{article}

\usepackage{EMNLP2022}
\usepackage{comment}
\usepackage{times}
\usepackage{latexsym}

\usepackage{amsfonts,booktabs}
\usepackage{bbold}
\usepackage[T1]{fontenc}
\usepackage[utf8]{inputenc}

\usepackage{microtype}

\usepackage{inconsolata}
\usepackage{graphicx,svg}

\usepackage{listings}
\usepackage{xcolor}

\definecolor{codegreen}{rgb}{0,0.6,0}
\definecolor{codegray}{rgb}{0.5,0.5,0.5}
\definecolor{codepurple}{rgb}{0.58,0,0.82}
\definecolor{backcolour}{rgb}{0.95,0.95,0.92}

\lstdefinestyle{mystyle}{
    backgroundcolor=\color{backcolour},   
    commentstyle=\color{codegreen},
    keywordstyle=\color{magenta},
    numberstyle=\tiny\color{codegray},
    stringstyle=\color{codepurple},
    basicstyle=\ttfamily\footnotesize,
    breakatwhitespace=false,         
    breaklines=true,                 
    captionpos=b,                    
    keepspaces=true,                 
    numbers=left,                    
    numbersep=5pt,                  
    showspaces=false,                
    showstringspaces=false,
    showtabs=false,                  
    tabsize=2
}

\usepackage{dsfont}
\usepackage{amsmath, amssymb, amsfonts, amsthm}
\usepackage{mathtools,commath}
\newtheorem{theorem}{Theorem}

\newtheorem{lemma}[theorem]{Lemma}

\allowdisplaybreaks

\title{SEAL: Interactive Tool for Systematic Error Analysis and Labeling}

\author{Nazneen Rajani\textsuperscript{$\dagger$}, Weixin Liang\textsuperscript{$\ddagger$}, Lingjiao Chen\textsuperscript{$\ddagger$}, Meg Mitchell\textsuperscript{$\dagger$}, James Zou\textsuperscript{$\ddagger$} \\
\textsuperscript{$\dagger$} Hugging Face
\quad
 \textsuperscript{$\ddagger$} Department of Computer Science, Stanford University \\
  \texttt{\{nazneen, meg\}@huggingface.co \quad
  \{wxliang,lingjiao,jamesz\}@stanford.edu} \\
  }

\newcommand{\seal}{{\sc{SEAL}}}

\theoremstyle{definition}
\newtheorem{definition}{Definition}[section]

\begin{document}
\maketitle
\begin{abstract}
With the advent of Transformers, large language models (LLMs) have saturated well-known NLP benchmarks and leaderboards with high aggregate performance. However, many times these models systematically fail on tail data or rare groups not obvious in aggregate evaluation. Identifying such problematic data groups is even more challenging when there are no explicit labels (e.g., ethnicity, gender, etc.) and further compounded for NLP datasets due to the lack of visual features to characterize failure modes (e.g., Asian males, animals indoors, waterbirds on land etc.). This paper introduces an interactive Systematic Error Analysis and Labeling (\seal) tool that uses a two-step approach to first identify high error slices of data and then in the second step introduce methods to give human-understandable semantics to those under-performing slices. We explore a variety of methods for coming up with coherent semantics for the error groups using language models for semantic labeling and a text-to-image model for generating visual features. \seal~toolkit and demo screencast is available at \url{https://huggingface.co/spaces/nazneen/seal}.

\end{abstract}

\section{Introduction}
\begin{figure}[ht]
    \centering
    \includegraphics[width=\linewidth]{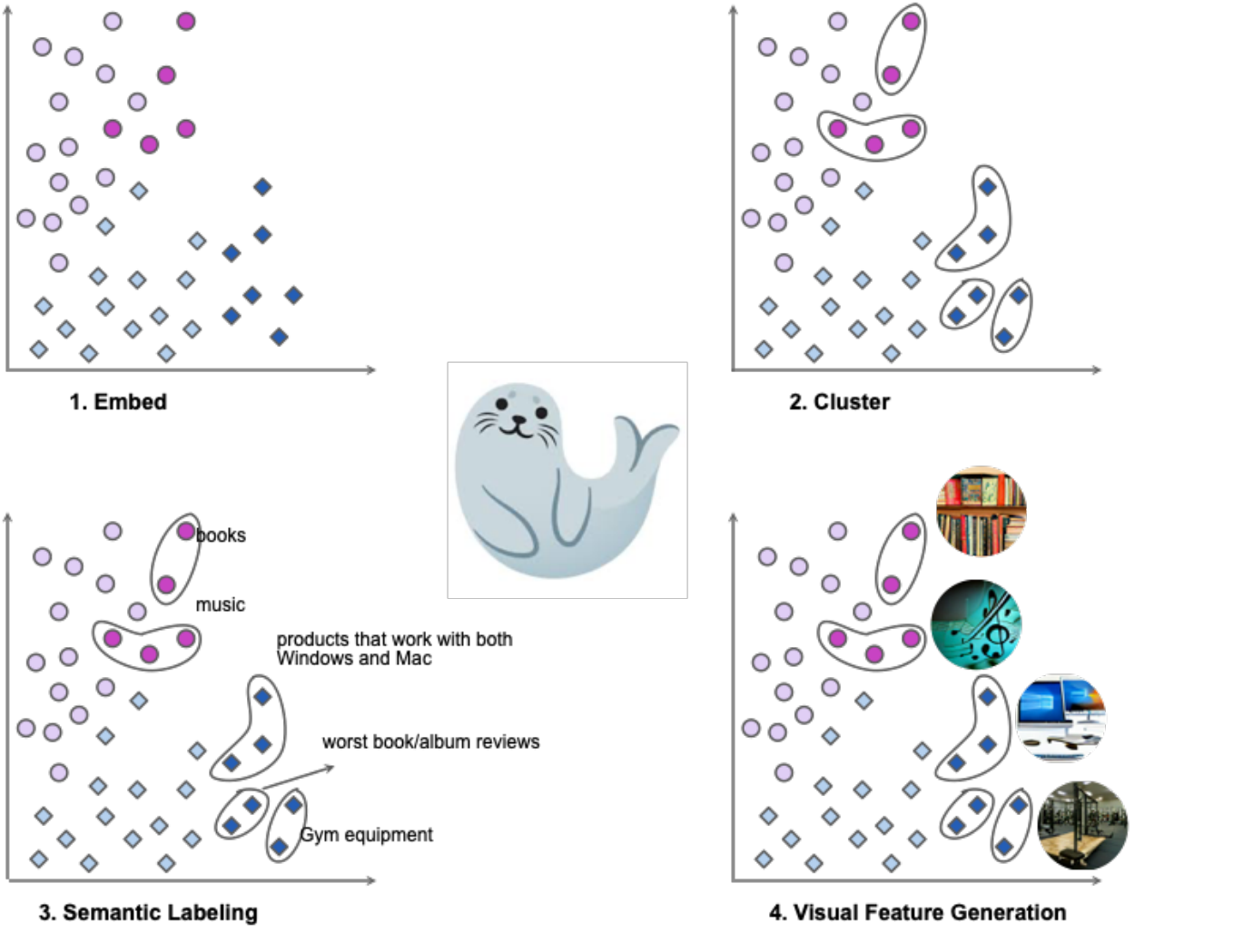}
    \caption{\seal~interactive tool for discovering systematic errors in model performance. Steps 1 and 2 include extracting the model embeddings and clustering datapoints with high-loss. Steps 3 and 4 include semantic labeling of error groups and generating visual features to support debugging.}
    \label{fig}
    \vspace{-10pt}
\end{figure}
Machine learning systems that seemingly perform well on average can still make \textit{systematic errors} on important subsets of data. Examples include such systems performing poorly for marginalized groups in chatbots~\citep{Chloe2018Microsoft}, recruiting tools~\citep{Isobel2018Amazon}, cloud products~\citep{Nicolas2020Google}, ad targeting~\citep{Karen2019Facebook}, credit services~\citep{Will2019Apple}, and image cropping~\citep{Isobel2020Twitter}. Discovering and labeling systematic errors in ML systems is an open research problem that would enable building robust models that generalize across subpopulations of data.

\begin{figure*}[t!]
    \centering
    \includegraphics[trim=0cm 0cm 2.5cm 0cm,, width=1.0\linewidth]{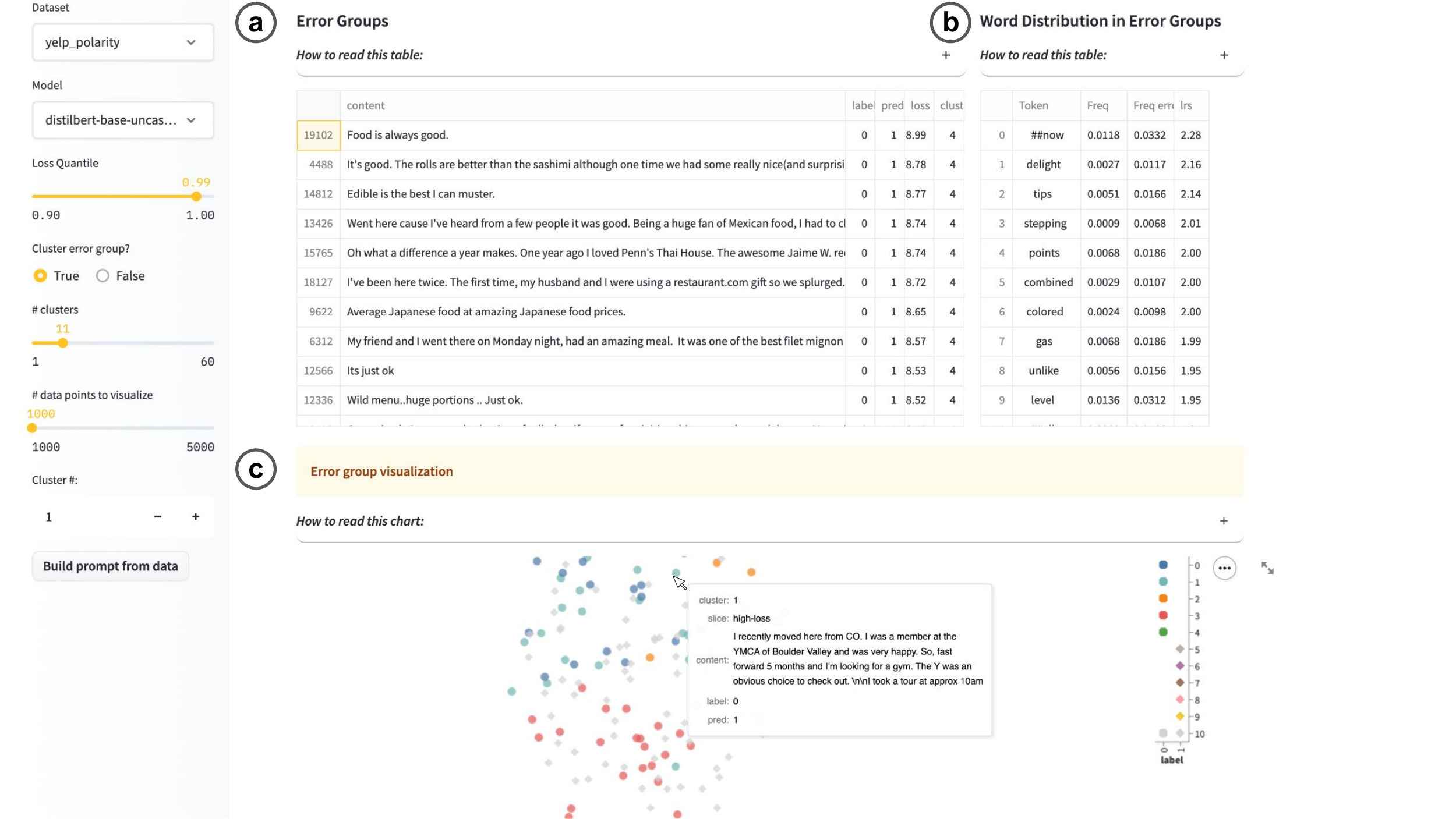}
    \caption{SEAL interface showing high-error groups for the distilbert-base-uncased model evaluated on the yelp\_polarity dataset. The interface comprises of various components: (a) examples from the dataset in the high error groups (sorted by loss), (b) statistics of tokens in high error groups relative to the entire evaluation set, (c) interactive 2d visualization of the model embeddings showing groups of errors in color and low-loss groups in gray. The colors indicate different error clusters. If the dataset has annotated classes, the visualization includes symbols to represents those classes ($\diamond$ and $\circ$ in the above figure). The panel on the left has multiple widgets that a user can control to be able to interactively understand their model's mispredictions relative to the rest of the model's outputs. Apart from the dataset and model, the user can select the loss quantile that want to examine for systematic errors, if they want SEAL to group those errors using kmeans++ with the number of clusters, and how many data points they want to visualize at a time in the visual component of the interface downsampled proportional to the group size (we use Altair for plotting that supports a maximum of 5000 data points to be visualized at once).}
    \label{seal-interface}
\end{figure*}

Uncovering underperforming groups of data of a ML system is not straightforward. Firstly, the high-dimensional space of the representations learned by the deep learning models makes it difficult to identify such groups of systematic errors. Secondly, it is difficult to extract and label the hidden semantic information in such groups with high errors without a human-in-the-loop setup. Identifying systematic model failures requires practitioners to think creatively about model evaluation~\citep{ribeiro-etal-2020-beyond, wu2019errudite, goel2021robustness, kiela-etal-2021-dynabench, Yuan_2022}. However, current approaches are mostly limited to examining and manipulating model mispredictions. The onus of identifying what group or subset of data to evaluate still falls on the practitioner, making it inefficient and prone to oversight. Recent works on fine-grained error analysis, such as Domino~\citep{eyuboglu2022domino} and Spotlight~\citep{d2022spotlight} provide solutions to this problem but focus on image datasets which are easier to visualize.

Error analysis for text data is less explored and more challenging. It also highlights the need to provide semantic summaries of text, which we tackle in SEAL. For example, NLP models could underperform on hundreds of possible input types -- longer inputs, inputs from non-native speaker, inputs with topic domains underrepresented in training, etc. This is a huge barrier of entry for most non-expert ML users who wish to gain a better understanding of their model and datasets with such existing tools. Model evaluation should ideally give actionable insights into a model's performance on a dataset in the form of data curation~\citep{liang2022metashift} or model patching~\citep{goel-etal-2021-goodwill}. 


Our desiderata is a tool that summarizes failures of a  model on textual data in a concise, coherent and human intepretable way. Systematic Error Analysis and Labeling (SEAL) is an interactive tool to 1. identify candidate groups of data with high systematic errors and 2. generate semantic labels for those groups. For 1, we use k-means++ on subset of evaluation data with highest loss. Semantic labeling uses LLMs (like GPT3) in zero-shot setting for identifying concepts or topics common to examples in the candidate group. We also explored using a text-to-image model to generate visual features for high error clusters using the Dall-e-mini~\citep{dalle}. Semantic descriptions (via labeling or visual features) of such systematic model errors not only enable practitioners to better understand the failure modes of their model during evaluation but also gives actionable insight to fix them via some form of model patching or data augmentation.


\section{SEAL}
We present Systematic Error Analysis and Labeling (\seal), an interactive visualization tool that provides rich data point comparison for text classification systems, enabling fine-grained understanding of model performance on data groups as shown in Figure~\ref{seal-interface}. It comes pre-loaded with model outputs
for most downloaded HuggingFace (HF) models and datasets, as well as scripts for loading data for any dataset provided by the Datasets API and extracting embeddings of any HF-compatible model.\footnote{Based on usage data from July'22 at \url{https://huggingface.co/models?pipeline_tag=text-classification&sort=downloads}}

\subsection{Error Discovery and Analysis}
Identifying model failures via error discovery is a crucial step in engineering robust systems that generalize to diverse subsets of data. SEAL uses the model's loss on a datapoint as a proxy for potential bugs or errors. Past work has examined model behavior on individual datapoints for mapping training datasets~\citep{swayamdipta-etal-2020-dataset}. We hope to leverage information about model behavior on individual \textit{evaluation} data-points in a similar fashion. We use quantiles for dividing the model loss region for further analysis. For example, Figure~\ref{seal-interface} shows the $0.99$ loss quantile for the distilbert-base-uncased model~\citep{Sanh2019DistilBERTAD} on the yelp\_polarity ~\citep{yelp} sentiment classification dataset.
The SEAL interface allows the user to control the loss quantile for fine-grained analysis using the widget on the side panel.

 SEAL uses k-means++ for clustering the high-loss candidate datapoints from the above step. \citet{meng2022topic} used k-means for topic discovery on the entire dataset and showed that the clusters are stable only when $k$ is very high ($k >> 100$) because of the scale of the embedding space. In contrast, SEAL only clusters the very high loss slice ($> 0.98$ quantile). 
 
 We use the representations of the models' final hidden layer (before the softmax) as embeddings. If the evaluation dataset selected by the user has ground truth annotations, then it groups the clusters by error-types (false-positives and false-negatives for binary classification). The visualization component of the SEAL interface shows the error clusters and their types using colors and symbols respectively. We use a standard heuristic of setting the number of clusters in k-means++ to be approximately $\sqrt{n/2}$, where $n$ is the group size.

\subsection{Semantic Error Labeling}
Semantic error labeling is important for identifying the underlying concept or topic connecting the datapoints in a error group. Systematic errors can be mathematically modeled and fixed by data curation. Contrast this with random errors that cannot be mathematically modeled or fixed via data curation. Past work analyzing NLP models have shown systematic errors on various tasks including sentiment classification, natural language inference, and reading comprehension ~\citep{mccoy-etal-2019-right, Kaushik2020Learning,jia-liang-2017-adversarial}. SEAL uses pretrained LLMs (such as GPT3~\citep{gpt} or Bloom~\citep{bloom}) for semantic labeling of error clusters that could highlight such possible systematic bugs in model performance. We craft a prompt consisting of  instruction and examples in the clusters extracted in the previous step as follows. 
\begin{lstlisting}[language=Python]
def build_prompt(content)
    instruction = 'In this task, we`ll assign a short and precise label to a group of documents based on the topics or concepts most relevant to these documents. The documents are all subsets of a ${task} dataset.'
    
    examples = '\n - '.join(content)
    
    prompt = instruction + '- ' + examples+ '\n Group label:'
    
    return prompt
\end{lstlisting}
Here \texttt{task} is the task under consideration for example `sentiment classification' in our case. The \texttt{arg} to the function is a dataframe or dataframe column with the dataset content as string that the model uses for classification.
Our prompt design was experimented first in the few-shot setting before adapting to the zero-shot.

\begin{table}[]
\small
\begin{tabular}{@{}lrr@{}}
\toprule
\textbf{Group label}&\textbf{Size}&\textbf{Group acc.}\\
\midrule
 \multicolumn{3}{c}{Albert Base v2 on Yelp (overall acc: 0.95)} \\ 
 \midrule
 Club reviews&574&0.90 (-5\%)\\
Movie theater reviews&231&0.85 (-10\%)\\
Dentist reviews&69&0.88 (-7\%)\\
Chain restaurant reviews&61&0.88 (-7\%)\\
Frozen custard reviews&37&0.83 (-12\%)\\
Waterfront business reviews&11&0.72 (-23\%)\\
 \midrule
\multicolumn{3}{c}{Distilbert Base Uncased on Amazon (overall acc: 0.89)} \\ 
\midrule
Bath product reviews&78&0.79 (-10\%)\\
Vaccuum cleaner reviews&34&0.76 (-13\%)\\
Eragon book reviews&28&0.67 (-22\%)\\
SD card reviews&13&0.61 (-28\%)\\
 \midrule
\multicolumn{3}{c}{Distillbert Base Uncased on IMDB (overall acc: 0.86)} \\ 
\midrule 
Reviews of movies starring `Bill'&644&0.79 (-7\%)\\
Adventure movie reviews&583&0.81 (-5\%)\\
Reviews of foreign films&262&0.80 (-6\%)\\
Movies with `stranger' in title&121&0.76 (-10\%)\\
Reviews of movies with pyschopaths&94&0.78 (-8\%)\\
Reviews of mystery movies&72&0.75 (-11\%)\\
\bottomrule
\end{tabular}
\caption{Results obtained from using SEAL on three sentiment classification datasets. The columns shows the group labels generated by GPT3, the size of the group in the overall evaluation set, and the group accuracy.}
\label{seal_results}
\end{table}

For the results and use case discussion in Section~\ref{results}, we use the OpenAI GPT3 API~\footnote{\url{https://beta.openai.com/playground}} via the CLI. The maximum token length is limited to $4000$ and so we truncate the prompt to that length before feeding the model.
We observed that for many larger groups of high-loss examples  ($>25$)~\seal~labels degenerate to generic output such as ``customer reviews of products'', ``movies reviews'', ``restaurant reviews'', etc. To prevent this and to generate coherent group labels, we sub-cluster the bigger error groups until their size is $<25$.  We verified the group labels by running the \citet{Blei03latentdirichlet} LDA topic model on the examples in each cluster after a pre-processing step. The pre-processing included tokenizing, lemmatizing, and removing stopwords. For each dataset domain, we also removed the domain word list -- (`movie, watch, film, character' for the IMDB dataset, `food, place, location, service, time, room, restaurant' for the Yelp dataset, and `book, author, pages, read, product' for the Amazon dataset). The concept tokens in the labels assigned by GPT3 were in the top-6 topics for these datasets.

\begin{figure}[ht]
    \centering
    \includegraphics[width=\linewidth]{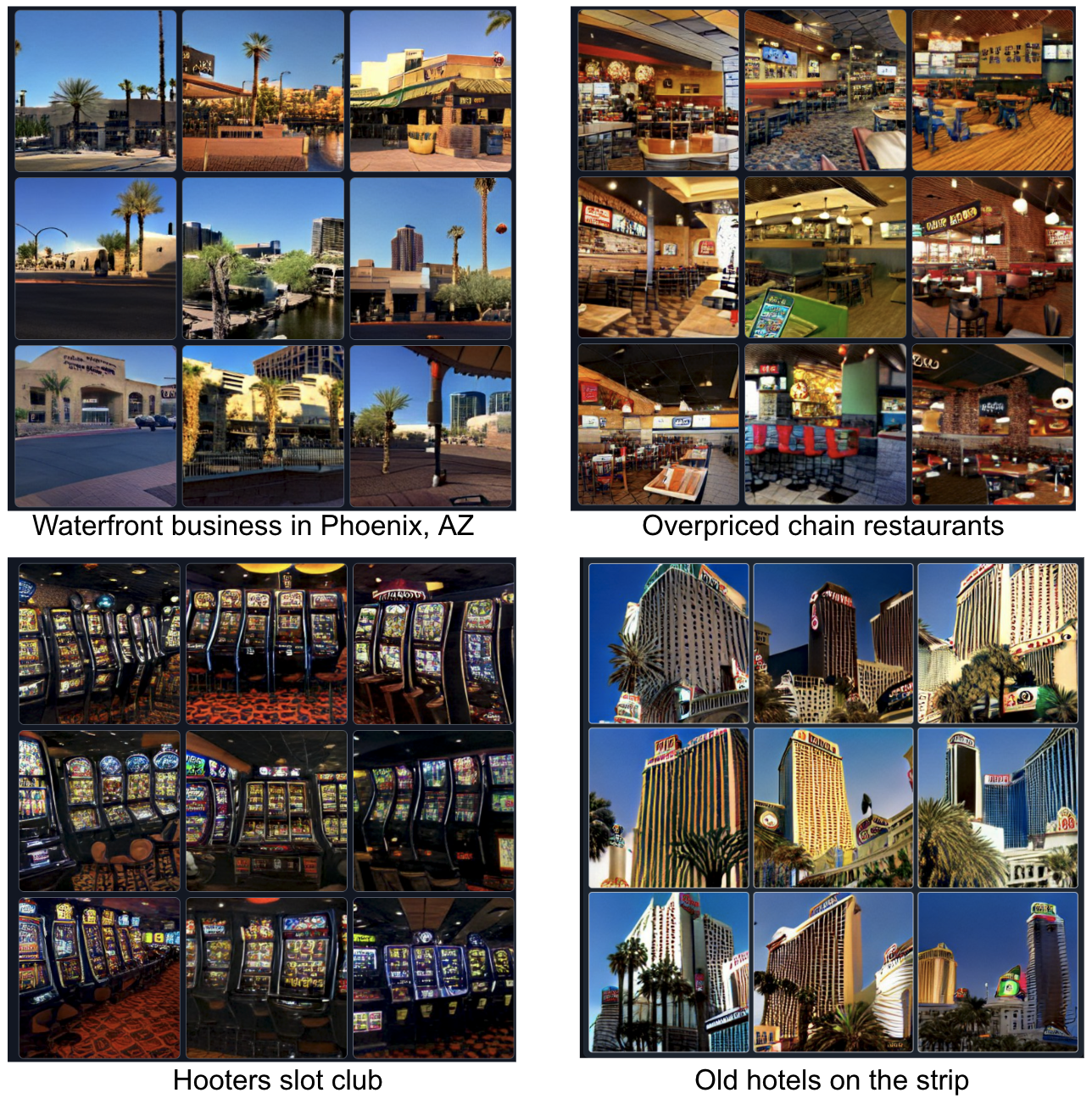}
    \caption{Examples of visualizations generated using Dalle-mini (Craiyon) for a sample of error groups.}
    \label{dalle}
\end{figure}

\seal~also supports querying the dalle-mini API to generate visual features that would support with error discovery.\footnote{\url{https://huggingface.co/spaces/dalle-mini/dalle-mini}} We augment the semantic labels generated using a LLM with the text-to-image diffusion model such as the dalle-mini. The goal is to further support systematic error discovery especially for users that are not domain experts in the dataset they are using. For example it is easy to imagine what `frozen custard' but it might not be obvious what `hooters slot club' is or what a `waterfront business in Phoenix, AZ' means. As shown in Figure~\ref{dalle}, the visual features help with further analysis and provide clear actionable insights. 
\subsection{System Architecture}
The interface is implemented as a Streamlit3 application with some customized HTML/JavaScript component that handles interactions in the tool. 
We use the Altair library customized with HTML/JavaScript and CSS for richer interactive visualization of embeddings. 
The visual component of the tool enables a user to interactively hover on data points and get information about the content, label, prediction, loss, and cluster (as in Figure~\ref{case-2}).
All the data preprocessing is powered by the Pandas library and all the manipulations on the data (such as extracting the layer embeddings, clustering, etc.) are stored as DataFrames thus providing a single interface for users to extend with custom data processing functions. 
We also provide pre-processing scripts to generate and cache all data required by~\seal~to ensure fast response times in the interface.
The scripts also include code to run inference (forward pass) on any HF dataset and model as well as a hook to extract learned representations from any layer of a loaded model. 
The workflow in~\seal~also enables users to interactively visualize data points with high loss using the streamlit slider widget to control the loss quantile that is highlighted on the interface.
\section{Results and Case Study}
In this section, we discuss some results using the~\seal~pipeline and walk through a case study for an interactive analysis with the tool.
\label{results}
\subsection{Experimental Results}
Table~\ref{seal_results} shows the results obtained using~\seal~on three sentiment classification datasets, Amazon~\citep{amazon}, Yelp~\citep{yelp}, and IMDB~\citep{imdb} for Distilbert~\citep{Sanh2019DistilBERTAD} and Albert~\citep{Lan2020ALBERT:}. For each dataset block in the table, we select the subset of group labels that were not generic (``customer reviews'', ``book reviews'') and either had proper names in them such as ``LensCrafters'', ``Eragon'' or  common nouns with properties such as ``trashy movies'', ``fine dining'', ``overpriced chain restaurants''.~\footnote{See relation /r/HasProperty \url{https://github.com/commonsense/conceptnet5/wiki/Relations}} We then measured model performance on all examples in the evaluation dataset that matched the group description to obtain the group accuracy.  Table~\ref{sample} shows the content for a random sample of examples in the error categories discovered using~\seal.

An unintended but interesting use case of~\seal~is to discover mislabeled candidate examples. We found that some groups have labels describing a sentiment such as ``trashy movies'', ``terrible food'' but with opposite ground truth sentiment. On further investigation, we found that indeed many of the groups have noisy labels and the model is actually predicting the correct sentiment. Table~\ref{appendixA} in the appendix shows a sample of such mislabeled candidate examples from each dataset studied in this paper.

\textbf{Limitations.} SEAL relies on the semantic robustness of the labeling LLM such as GPT3. We did not test cluster labeling on NLP tasks that require understanding semantic phenomena or function word.
\begin{table*}[t!]
\scriptsize
\begin{tabular}{@{}p{0.10\textwidth}p{0.75\textwidth}p{0.05\textwidth}p{0.05\textwidth}@{}}
\toprule
\textbf{Group label}&\textbf{Content}&\textbf{Label}&\textbf{Pred}\\
\midrule
Club reviews&Being from Southern California, the ``scene'' is so much fun. There are several clubs to go to and any night is a great time. That brings us to the Phoenix scene and The Cash.Oh wait, there is no scene for the ladies. Not going to bash them to hard, because it's the only consistent place that we have. Yes it caters to the Country music crowd, but they do play spurts of other music through out the weekend evenings.The mixed drinks could be better, but the prices are reasonable.&0&1\\
\\
&I used to come here for years, maybe about a year back.. the best weekend drinkfests back then: Fridays were ladies night (dollar well, wines and domestics, \$2 you call its, and no cover). Saturdays were free beer night (draft bud light, coors light and pbr til they gave out 1,000 of each.. again, no cover). Was always packed and played a decent variety of music; pitchers for beer pong were also always dirt cheap. And despite, the bartenders were way personable and fun.I'm not trying to sound like a cheapskate, as I am in the service industry myself.. but there must've been a change of ownership since my prior experiences.[..]&0&1\\
\\
Dentist reviews&Thank you for all the emails you sent me on my review! I was surprised at how many responses I recieved from people searching for the right dentist..I shared my new dentist information and even got some movie tickets from my dentist for the referrals! I find it funny how since I wrote this review how many people have reviewed with 5 stars... They must have a lot of friends and family! I hope everyone reads my review and picks the right dentist for your needs! Happy Holidays&0&1\\
\\
&After dealing with a two week long migraine and severe pressure and pain in my face, I called around looking for an ENT that could get me in ASAP. Dr. Simms was available for a same day appointment and I scheduled with him for that afternoon. The wait time itself wasn't bad - 10-15 minutes after completing paperwork. Dr. Simms was personable enough and after evaluating me, told me that he would like to treat for a sinus infection with antibiotics and prednisone. As I had just moved and newly became a student, I didn't yet have health insurance set up.[..] &0&1\\
\\
Chain restaurant reviews&I tried Cozymel's on a recommendation from my parents. Living in San Diego, I never go to chain Mexican places - there are just too many other places to try. I was expecting Cozymel's to be okay, nothing great.We went for lunch, and I was happy to see a whole page of lunch specials for about \$8. Usually, an enchilada combo plate could set you back close to \$15 at a Mexican chain. Not here (during lunch at least). I ordered the taco salad with black beans instead of meat. It came in an enormous flour tortilla shell - tostada style.[..]&1&0\\
\\
&I still can't get over how I paid \$2.99 for a coffee and 3 doughnuts! What a deal. I was debating whether or not to go to Krispy Kreme or Winchells but decided on the latter since it wasn't a chain and I could get Krispy Kreme elsewhere...Winchell's shares space with Subway which was a little random but I didn't have any problem with it because the woman helping me and what I assume to be the owner were both very nice and sweet. I hadn't eaten doughnuts in a little over a year so I decided to go with a boston creme (one of my favorites) and got a chocolate glazed chocolate doughnut for my sister and a glazed for my friend.[..]&1&0\\
\\
 \bottomrule
\end{tabular}%
\caption{Random sample from under-performing groups discovered by SEAL for the Yelp dataset. Results for other datasets are in Table~\ref{sample_Appendix} in the appendix. 0 and 1 indicate negative and positive sentiment classes respectively. The reviews ending in [..] have been truncated to save space.}
\label{sample}
\end{table*}

\begin{figure}[t]
    \centering
    \includegraphics[width=\linewidth]{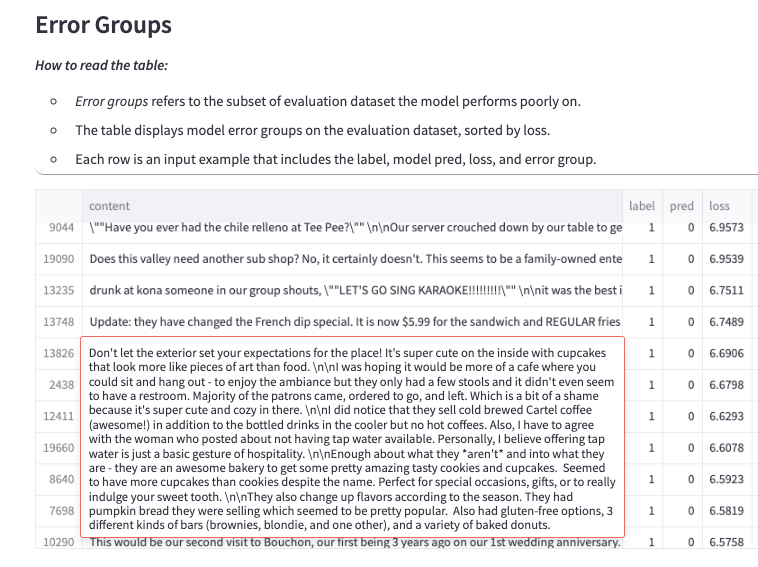}
    \caption{Snapshot of~\seal~showing the table of examples with highest-loss and their clusters.}
    \label{case-1}
\end{figure}

\begin{figure}[ht]
    \centering
    \vspace{-30pt}
    \includegraphics[width=0.75\linewidth]{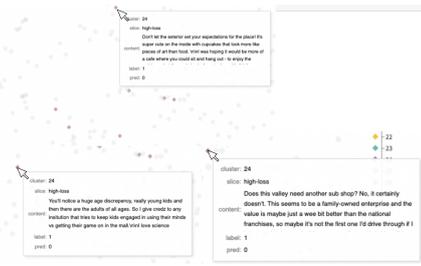}
    \caption{Snapshot from the~\seal~interface highlighting a group of examples with high-loss that are candidates for a systematic error type where reviews consist of customer experiences being better than their expectation of the place.}
    \label{case-2}
\end{figure}

\subsection{Case Study}
\seal~with its interactive interface enables practitioners to discover possible systematic errors in their models. In this section, we walk-through a case study of identifying such errors with the Albert-base-v2 model finetuned and evaluated on the Yelp dataset.
The user first loads the model and dataset in the tool and examines the examples with the highest-loss as in Figure~\ref{case-1}. They notice that the example includes  customer reviews where there was discrepancy between expectation and reality. They then want to zoom in to figure out similar reviews in the dataset where customers experiences differed from their expectations. They run the clustering and visualize the high-loss examples interactively. After trying a few values of `\# of clusters', the user finds that indeed there are many other such examples that surface in the visualization component of~\seal~as shown in Figure~\ref{case-2}. The model underperforms on examples of the type where the customer expectation is negative but the reality is actually positive.

\section{Mathematical robustness of SEAL}
\label{sec:theory}

In this section, we provide theoretical guarantees for the stability of semantic labels generated by the \seal{} pipeline. More specifically, our stability theorem states that a small perturbation of the input of our \seal{} pipeline would only cause a small bounded difference of the semantic labels. 
An implication of our theoretical results is that, even if two users are using different versions of an evaluation set (e.g., a different split, or a smaller subset), \seal{} would generate similar semantic labels. 

More formally, we ask: How does a small change in the input dataset $\{(x_i,y_i)\}_{i=1}^{n}$ affect the semantic label tuple $M \triangleq \{m_k\}_{k=1}^{K}$? 
Here, $K$ denotes the number of explanations, $m_k\triangleq(w_k,s_k,a_k)$ encodes the $k$th explanation message, 
where $w_k, s_k$, and $a_k$ represent the sentence vector, the number of data points explained by this message, and the average accuracy among those data points. Here we show that under some assumptions, the outputs of SEAL, i.e. the set of $m_k$, is relatively robust to randomness in the input dataset. To be more precise,  we need  a distance metric  on explanation message space. 
\begin{definition}
Given any two semantic label tuple $M=\{m_k\}_{k=1}^{K}$ and $M'=\{m_k'\}_{k=1}^{K}$, define a distance $d_{\max}(M,M')$ between them as 
\begin{equation*}
 \max_{1\leq i\leq K} \min_{1 \leq j \leq K} \|m_i-m'_j\|_2 + \|m'_i-m_j\|_2 
\end{equation*}
\end{definition}
\paragraph{Remark.} The $\ell_2$ distance $\|\|_2$ is defined on the vectorized explanation. In other words, we concatenate the sentence vector, data point number, and the accuracy value in one single vector, and then measure the distance of two explanation messages by the distance of their corresponding expanded vectors.  

Here, a small  distance value $d_{\max}$ implies a small difference in the explanation word vector, the size of each cluster, and the accuracy within each cluster.
To see this, note that a small distance implies that for any messages $m_i$ and $m_j$ in $M$, one can find two other messages $m_i'$ and $m_j'$ in $M'$, which are close to them. That is to say, each for any message in $M$, there is a message in $M'$ approximately equal to it.
Now we can answer the raised question.

\begin{theorem}\label{thm:seal:robustness}
Let $S$ and $T$ denote two set of $n$ data points i.i.d. from some data distribution $P$. Suppose the probability space of $P$ is  compact with size $B$, and the density function is bounded. 
Let $M_S$ and $M_T$ be the semantic label tuples generated by \seal{} with input $S$ and $T$. 
If $S$ and $T$ differs in $o(\sqrt{n})$ data points, and the the clustering algorithm  gives the exactly optimal solution, then we have 
\begin{equation*}
d_{\max}(M_S, M_T) \xrightarrow[]{P} 0,     
\end{equation*}
i.e., $d_{\max}(M_S, M_T)$ converges to $0$ in probability.
\end{theorem}
The proof of this theorem is in the Appendix. It implicitly relies on Lipschitz continuity of the sentence generation network, which actually holds for most DNNs with finite input space.
This indicates \seal{} is robust to small perturbation in the input dataset: a small shift in the input dataset only leads to small explanation change. Such a smooth explanation change is particularly useful when users gradually update the their dataset.
\section{Conclusion}
In this work we introduced~\seal, an interactive visualization tool for discovering systematic errors and labeling them.
Through case studies we showed how SEAL can  efficiently identify the systematic failures of state-of-the-art sentiment classification models on well known datasets.
We released a set of pre-computed model outputs to enable easy, out-of-the-box use especially for non-coding audience such as domain experts.
We hope this work will positively contribute to the ongoing efforts in building tools for systematic error analysis and model debugging.
\section{Ethics Statement}
Many datasets currently used and open-sourced by the NLP community are mainly crawled from the web and therefore are not representative of a majority of geographies. There are biases that can distill into parameters of models trained on such biased datasets and may even be further amplified in the generated model outputs.  All datasets we experimented with are in English, and all models are trained on English datasets.

We use GPT3 for semantic labeling and it is well-known that LLMs such as GPT3 can generate toxic, harmful, hate content that might have also percolated into our tool. Similarly, the semantic similarity metrics used in our tool including the BERTScore and the word-embeddings carry biases of the data they were trained on. 
We request our users to be aware of these ethical issues that might affect their analyses.

\section*{Acknowledgements}
We thank the anonymous reviewers for their constructive feedback. We are also thankful to David McClure and Christopher Akiki whose work inspired some of the interactive components in the demo accompanying this paper.
\bibliography{custom}
\bibliographystyle{acl_natbib}

\appendix
\newpage
\onecolumn
\label{sec:appendix}


\section{Appendix: Proofs}
\label{appendix:sec:proofs}
\begin{proof}
Here we prove the proof for Theorem \ref{thm:seal:robustness}.
To proceed, we need a few  lemmas.

\begin{lemma}[adapted from Proposition 5.1. in \cite{rakhlin2006stability}]
Assume the density of $P$ (with respect to the Lebesgue measure $\lambda$ over $\mathcal{Z}$ ) is bounded away from 0, i.e. $d P>\mu d \lambda$ for some $\mu>0$. Suppose the clusterings $A$ and $B$ are minimizers of the $K$-means objective $W(C)$ over the sets $S$ and $T$, respectively.  Suppose that at most $ o(\sqrt{n})$ data points are different between the two dataset $S$ and $T$ sampled from $P$. Then
$$
d_{\max }\left(\left\{c_{S,1}, \ldots, c_{S,K}\right\},\left\{c_{T,1}, \ldots, c_{T,K}\right\}\right) \stackrel{P}{\longrightarrow} 0 .
$$
where $c_{S,i}$ and $c_{T,i}$ are the centers of the $i$-th cluster generated from $S$ and $T$, separately. 
\end{lemma}


\begin{lemma}
Assume the density of $P$ (with respect to the Lebesgue measure $\lambda$ over $\mathcal{Z}$ ) is bounded away from 0, i.e. $d P> \mu d \lambda$ for some $\mu>0$. Suppose
$$
d_{\max }\left(\left\{c_{S,1}, \ldots, c_{S,K}\right\},\left\{c_{T,1}, \ldots, c_{T,K}\right\}\right) \leq \varepsilon .
$$
and the ML model that generates the sentence vector is Lipschitz continuous with parameter $\beta$.
Then
$$
d_{\max }\left(M_S,M_K\right) \leq 3 \varepsilon \max \{6  K^2B,\beta \}
$$
where $c_{c, m}$ depends only on $c$ and $m$.
\end{lemma}
\begin{proof}
We first note that, by triangle inequality, we have 
\begin{equation*}
\begin{split}
&d_{\max} (M_S,M_K)\\ =&\max_{1\leq i\leq K} \min_{1 \leq j \leq K} \|m_{S,i}-m_{T,j}\|_2 + \|m_{S,j}-m_{T,i}\|_2  \\
\leq & \max_{1\leq i\leq K} \min_{1 \leq j \leq K} \|w_{S,i}-w_{T,j}\|_2 + \|w_{S,j}-w_{T,i}\|_2 +\|s_{S,i}-s_{T,j}\|_2 \\
+&\|s_{S,j}-s_{T,i}\|_2 +\|a_{S,i}-a_{T,j}\|_2 + \|a_{S,j}-a_{T,i}\|_2 
\end{split}
\end{equation*}
Note that, by $\min_j\{a_j+b_j+c_j\}\leq \max \{3 \min_j a_j, 3 \min_j b_j, 3 \min_j c_j\}$, the inner minimization is bounded by 3 times the maximum of  $\min_{1 \leq j \leq K} \|w_{S,i}-w_{T,j}\|_2 + \|w_{S,j}-w_{T,i}\|_2$, $\min_{1 \leq j \leq K} \|s_{S,i}-s_{T,j}\|_2 + \|s_{S,j}-s_{T,i}\|_2$, $\min_{1 \leq j \leq K} \|a_{S,i}-a_{T,j}\|_2 + \|a_{S,j}-a_{T,i}\|_2$.
Now let us consider those terms separately:
\begin{enumerate}
    \item $\min_j \|w_{S,i}-w_{T,j}\|_2 + \|w_{S,j}-w_{T,i}\|_2 $: By Lipschitz continuity, the distance between two sentence vectors can be bounded by the distance between their corresponding cluster centers.
More precisely, 
$$\|w_{S,i}-w_{T,j}\|_2 + \|w_{S,j}-w_{T,i}\|_2 \leq \beta \|c_{S,i}-c_{T,j}\|_2 + \beta \|c_{S,j}-c_{T,i}\|_2$$ and thus
$$\min_j \|w_{S,i}-w_{T,j}\|_2 + \|w_{S,j}-w_{T,i}\|_2 \leq \beta \min_j \|c_{S,i}-c_{T,j}\|_2 +  \|c_{S,j}-c_{T,i}\|_2$$
By the assumption, the right hand side is bounded by $\varepsilon$, and thus 
$$\min_j \|w_{S,i}-w_{T,j}\|_2 + \|w_{S,j}-w_{T,i}\|_2 \leq \beta \varepsilon$$

\item $\min_j \|s_{S,i}-s_{T,j}\|_2 + \|s_{S,j}-s_{T,i}\|_2 $:
By the assumption, we know that, for any given $i$, we can find $j$, such that $\|c_{S,i}-c_{T,j}\| + \|c_{S,j}-c_{T,i}\|\leq \varepsilon$. 
That is to say, the cluster centers' distance is at most $\epsilon$.
Since the distribution space is bounded by $B$, there are at most $\varepsilon$, there are at most $2 \epsilon B$ data points are clustered differently. 
As there are $K$ clusters, in total at most $2 \epsilon K^2 B$ data points are clustered differently. 
This gives a natural upper bound 
$$\min_j \|s_{S,i}-s_{T,j}\|_2 + \|s_{S,j}-s_{T,i}\|_2 \leq 6 \varepsilon K^2 B$$

\item $\min \|a_{S,i}-a_{T,j}\|_2 + \|a_{S,j}-a_{T,i}\|_2 $: Now applying a similar argument in 2, we know that 
in total at most $2 \epsilon K^2 B$ data points are clustered differently. 
Thus, at most 
$2 \epsilon K^2 B$ data points affect the accuracy value. This means
$$\min_j \|a_{S,i}-a_{T,j}\|_2 + \|a_{S,j}-a_{T,i}\|_2 \leq 6 \varepsilon K^2 B$$

\end{enumerate}
Combining those results, we can conclude that 
\begin{equation*}
\begin{split}
& \min_{1 \leq j \leq K} \|w_{S,i}-w_{T,j}\|_2 + \|w_{S,j}-w_{T,i}\|_2 +\|s_{S,i}-s_{T,j}\|_2 \\
+&\|s_{S,j}-s_{T,i}\|_2 +\|a_{S,i}-a_{T,j}\|_2 + \|a_{S,j}-a_{T,i}\|_2 \\
\leq & 3 \max \{6 \varepsilon K^2B,\beta \varepsilon\}
\end{split}
\end{equation*}
This is independent of $i$, and thus we can take the maximum over $i$, which gives 
\begin{equation*}
\begin{split}
& \max_i \min_{1 \leq j \leq K} \|w_{S,i}-w_{T,j}\|_2 + \|w_{S,j}-w_{T,i}\|_2 +\|s_{S,i}-s_{T,j}\|_2 \\
+&\|s_{S,j}-s_{T,i}\|_2 +\|a_{S,i}-a_{T,j}\|_2 + \|a_{S,j}-a_{T,i}\|_2 \\
\leq & 3 \max \{6 \varepsilon K^2B,\beta \varepsilon\}
\end{split}
\end{equation*}
That is,
\begin{equation*}
\begin{split}
d_{\max}(M_S,M_T) 
\leq & 3 \max \{6 \varepsilon K^2B,\beta \varepsilon\}
\end{split}
\end{equation*}
which completes the proof.
\end{proof}

Combining the above two lemmas directly proves the robustness statement.
\end{proof}

\section{Appendix: More Examples}
\begin{table*}[ht]
\scriptsize
\begin{tabular}{@{}p{0.15\textwidth}p{0.85\textwidth}@{}}
\toprule
\textbf{Group label}&\textbf{Group content sample}\\
\midrule
\multicolumn{2}{c}{Amazon}\\
\midrule
 Customer reviews for a product that has been discontinued& Another reviewer recently advised that this is the model to look for. I was just advised at a well known retailer that this model has been discontinued. Is this true or is this a classic bait-and-switch technique? Their current weekly sales circular features this model at a sale price. When you get to the store, they don't have it but when they look it up in their computer, it shows up as "Discontinued". It is difficult to relate reviews to actual products when the reviews you base your buying decision on could be about(a)different model(s) from the one you actually buy online or in-store. The Creative Labs' own website does not give model numbers so they are adding to the confusion. \\
\\
&The software mentioned on my May 16th review IS called "AVID Xpress" -- not "AVID Express" -- when my review was edited someone changed the spelling, possibly thinking it was a typo/mistake?\\
\\
&Although this show is very fascinating I find every episode to be almost the same. Starting with Morgan Freeman stating "when I was a young boy..." then something he did to get in trouble, or something he witnessed that ruined his fragile eggshell mind. Followed by rhetorical questions and theories, and tons and tons of examples. The examples even have examples. Maybe I just understand this stuff and the show really dumbs it down, but I feel like I wasted money investing in season 3. Which by the way, although not currently available, (I don't have cable and I still had the privilege of watching this before the DVD came out) but I will still probably end up buying it on DVD which is cheaper than I already paid for the electronic proprietary/DRM version on Amazon Unbox\\
\\
Unreliable book reviews&I do not intend to review content here. This new edition is so full of typographical errors that sometimes the reader will have to intuit what the author really wrote. It is clear that the proofreaders of this edition were not actually reading; they were simply following the little red lines under the "misspelled" words. This has resulted in some truly bizarre apparent statements by the author, unreproducible here due to copyright laws. Disclaimer-- I have not purchased this book, merely checked it out of the library.\\
\\
&It's been several years since I've read "Silent Spring," one of the most significant environmental books ever written, but I must respond to the posting by "seem," which is titled "murderous, over the top propaganda" (I correctly your misspelling of the last word): His recommendation to read "DDT: A Case Study in Scientific Fraud" was put out by the Heartland Institute and is, in itself, a "fraud." The Heartland Institute is one of the most pro-chemical, pro-industry, anti-environmental and right-wing organizations around. Nothing they put out should be believed for a second.\\
\\
&Shame on all the booksellers selling this ten dollar book for \$75 and up!Devorss is re-publishing this book in August!I took note of the sellers AND WILL NEVER BUY FROM THEM!\\
\\
\midrule
\multicolumn{2}{c}{Yelp}\\
\midrule
 Terrible dry cleaners in Phoenix& I went here for the first time on First Fridays, yeah so what. I promise that I won't hang out here all the time and ruin it for all you true Bikini lovers. My mini pitcher was \$3.50 and then 5 minutes later a chick walked up and got charged \$6.00 for two mini pitchers, hmmm, male discrimination or they can't do simple math? I'll only go back when it's 110 outside and want to put a buzz on early in the afternoon.\\
 \\
 &Mediocre dry cleaning. I want to like this business..why? 1. I like to support Yelp advertisers 2.prime location!!! It is literally around the corner from me and I will probably still go there once in awhile out of convenience. Once or twice I called rushing to get there before they closed and they waited a minute over closing time which was very nice of them. However, this review is simply based off of satisfaction with my clothing. Almost every time I have come there I have to ask to redo my shirts. It drive me nuts because the employees are nice about it. When I woke up today and had 50 dollars worth of clothing needed to be dry cleaned I drove 20 minutes to my old favorite cleaners in Arcadia. I knew that I trust them with my clothes and after years, never had to deal with such an inconvenience. I'm sorry but had to only give 3 stars. I might be back one more time.. only when I have to. John I read your message and appreciate that so I updated my review out of appreciation towards your response. I want to come back because it is convenient. Thanks for caring  \\
 \\
 & This place is tiny and has more high-end expensive beads than other stores in Phoenix. I've found some really special items here. You shouldn't expect to buy more than a few strands at a time, as it just isn't affordable. Go somewhere else for quantity, and just get a few things to spice up your mix from Bead World. \\
 \\
 \midrule
\multicolumn{2}{c}{IMDB}\\
\midrule
 Terrible movies&You have to be awfully patient to sit through a film with one-liners so flat and unfunny that you wonder what all the fuss was about when WHISTLING IN THE DARK opened to such an enthusiastic greeting from audiences in the 1940s.<br /><br />On top of some weak one-liners and ordinary sight gags, the plot is as far-fetched as the tales The Fox (Red Skelton) tells his radio audience. You have to wonder why anyone would think he could come up with a real-life solution on how to commit the perfect crime and get away with it. But then, that's how unrealistic the comedy is.<br /><br />But--if you're a true Red Skelton fan and enjoy a look back at how comedies were made in the '40s--you can at least enjoy the amiable cast supporting him. Ann Rutherford and Virginia Grey do nicely as his love interest and Conrad Veidt, as always, makes an interesting villain. One of his more amusing moments is his reaction to Skelton explaining the mysteries of wearing turbans. "I never knew that," he muses, impressed by a minor point that is cleverly introduced.<br /><br />All in all, typical nonsense that requires you to accept the lack of credibility and just accept the gags as they are. Not always easy for a discriminating viewer as many of them simply fall flat, the way many comedies of this era do because the novelty of the sight gags and one-liners has simply worn off.\\
 \\
 &If they gave out awards for the most depraved and messed-up movies in the world, Japanese cinema would clean up: their exploitation cinema wipes the floor with most other contenders, the most extreme examples being absolutely jaw-dropping exercises in bad taste, nauseating gore, freakish weirdness, and misogynistic sex.<br /><br />Guts of a Beauty is a prime example of such whacked out filth, offering discerning viewers just over an hour of full-on debauchery and gratuitous violence topped off with some very insane J-splatter goodness.<br /><br />The film opens with a young woman named Yoshimi, whose search for her missing sister has led her into the hands of some nasty yakuza, who proceed to rape her and shoot her full of strong dope called Angel Rain[...]\\
 \\
 European Union movie is disappointing and full of clichés& **SPOILERS AHEAD**<br /><br />It is really unfortunate that a movie so well produced turns out to be<br /><br />such a disappointment. I thought this was full of (silly) clichés and<br /><br />that it basically tried to hard. <br /><br />To the (American) guys out there: how many of you spend your<br /><br />time jumping on your girlfriend's bed and making monkey<br /><br />sounds? To the (married) girls: how many of you have suddenly<br /><br />gone from prudes to nymphos overnight--but not with your<br /><br />husband? To the French: would you really ask about someone<br /><br />being "à la fac" when you know they don't speak French? Wouldn't<br /><br />you use a more common word like "université"? <br /><br />I lived in France for a while and I sort of do know and understand[...] \\ 
 \\
 &Obviously made on the cheap to capitalize on the notorious "Mandingo," this crassly pandering hunk of blithely rancid Italian sexploitation junk really pours on the sordid stuff with a commendable lack of taste and restraint: The evil arrogant white family who own and operate a lavish slave plantation spend a majority of the screen time engaging in hanky panky both each other and their various slaves[...]\\
 \\
 \bottomrule
\end{tabular}%
\caption{Mislabeled candidate examples for the three sentiment classfication datasets. All the examples have GT as positive. Examples ending in [..] have been truncated to save space.}
\label{appendixA}
\end{table*}

\begin{table*}[ht]
\scriptsize
\begin{tabular}{@{}p{0.10\textwidth}p{0.80\textwidth}p{0.05\textwidth}p{0.05\textwidth}@{}}
\toprule
\textbf{Group label}&\textbf{Content}&\textbf{Label}&\textbf{Pred}\\
\midrule
Reviews of mystery movies&Based on a Stephen King novel, NEEDFUL THINGS provides the intrigue and eeriness to keep you in your seat. A mysterious man(Max von Sydow) comes to town and soon becomes the most talked about citizen. Could it be that the devil himself has set up shop as an antique dealer in a small town in Maine? von Sydow is masterful and dynamic in this role that dominates the screen. Also starring are Ed Harris and Bonnie Bedelia. Harris is steady and Bedelia is deserving of your attention. Also in support are J.T. Walsh and Amanda Plummer. Not the best, nor the worst adaptation of King's horror on the screen.&0&1\\
\\
&Before I begin, let me get something off my chest: I'm a huge fan of John Eyres' first film PROJECT: SHADOWCHASER. The film, a B-grade cross of both THE TERMINATOR \& DIE HARD, may not be the work of a cinematic genius, but is a hugely entertaining action film that became a cult hit (\& spawned two sequels \& a spin off). Judge and Jury begins with Joseph Meeker, a convicted killer who was sent to Death Row following his capture after the so-called "Bloody Shootout" (which seems like a poor name for a killing spree. Meeker kills three people while trying to rob a convenience store), being led to the electric chair. There is an amusing scene where Meeker talks to the priest about living for sex but meeting his one true love (who was killed during the shootout), expressing his revenge for the person who killed her. Michael Silvano, a washed-up football star who spends his days watching his son Alex practicing football with his high school team (and ends up harassing his son's coach). But once executed, Meeker returns as a revenant (or as Kelly Perine calls "a hamburger without the fries")[..]&0&1\\
\\
&Let me say this about Edward D. Wood Jr. He had a passion for his work that I wish more people did have. If we all had the optimism and the commanding hope of Ed Wood, the world would probably be a much better place. Being familiar with Ed Wood's story and having seen the most wonderful biopic "Ed Wood" (1994) several times, I admire his boldness and his strives for the job he loved; I still admire his never-say-die attitude. He had a love for directing that I wish more people in modern-day Hollywood had.But that doesn't make his movies any more fun to watch. And "Glen or Glenda," his first and most confessional film, is probably his very worst."Glen or Glenda" is a deadening cult movie about a cross-dresser named Glen (played by director/writer Ed Wood himself) who despite his love for his fiancée Barbara (Dolores Fuller), cannot seem to conquer his lust for transvestitism, in which he dresses in women's clothing and a wig and thus becomes...Glenda! Glen/Glenda's story is narrated by a doctor and he too is talked and watched over by a mysterious character called "The Scientist" played by veteran horror star Bela Lugosi.[..]&0&1\\
\\
Adventure movie reviews&Just exactly HOW director John Madden come to settle with Nicolas Cage and Penelope Cruz playing the roles of an Italian Officer and a Greek Villager in an honourable story: "Captain Correli´s Mandolin", just escapes me! Witness: a wobbly, inconsistent accent by Cage amid horrendous over-acting, with Cruz -- more adequately cast as a spoiled Latino opposite Johnny Depp in "Blow" -- in basically a repeat performance under the guise of a Greek nurse... ay, it was painful. But there were saving graces.The story itself is thrilling-to-tragic, and Cage does have some (-- redeeming, this is !--) musical ability. Next, a superb performance by John Hurt (Cruz´s father, the village doctor) of Oscar Callibre, as well as by Irene Papas, each as village elders, as well as by Christian Bale (Papas´ son) among the village freedom fighters, go far towards counter-balancing awkward performances (especially at the beginning) by Cruz and Cage. Nicely, the last two seem to grow into their respective roles as the film progresses, but it´s teeth-gnashing early on. Finally, the scenery itself and the photography could garner a technical award, and such provides pleasant distractions when most needed.[..]&1&0\\
\\
&Daisy Movie Review By James Mudge From beyondhollywood.com. On paper, "Daisy" sounds like an Asian film fan's dream come true, directed by "Infernal Affairs" co-helmer Andrew Lau and starring everybody's favourite sassy girl, popular Korean actress Jeon Ji Hyun. Unfortunately, despite the talent involved, and the fact that the crew flew halfway around the world to shoot in Amsterdam , the film turns out to be a bit of a disappointment, being a clich'd romantic drama which wallows in misery and self importance. The plot follows Hye Young (Jeon Ji Hyun), a rather naive Korean girl who lives in Amsterdam , spending her life working in her grandfather's antique shop and doing portraits for tourists. One day, she begins receiving flowers at exactly the same time from a secret admirer, who she believes to be a mystery man from her past who once built her a nice little bridge. One day she meets Jeong Woo (Lee Seong Jae, also in "Holiday" and "Public Enemy"), who unbeknownst to her is actually an Interpol agent tracking Asian criminals in the Netherlands .With Hye Young assuming that Jeong Woo is responsible for the flowers, the two fall very slowly into a chaste romantic relationship. However, it turns out that the man sending the flowers is actually Park Yi (Jung Woo Sung, from "Sad Movie" and "Musa"), an assassin working for a Chinese crime syndicate. Inevitably, the love triangle turns tragic and the two men end up facing off while poor Hye Young tries to work out which of the two is the love of her life.Although "Daisy" is ostensibly a love story, it has the feel of a funeral, with a slow, sombre pace and a plot which piles on the misery. Half of the film's running time is taken up with scenes of the characters staring longingly out of windows into the rain, with the silence broken only by bouts of self pitying narration.[..] &1&0\\
\\

 \bottomrule
\end{tabular}%
\caption{Random sample from under-performing groups discovered by SEAL for the IMDB dataset. 0 and 1 indicates negative and positive sentiment classes respectively.}
\label{sample_Appendix}
\end{table*}


\end{document}